\RequirePackage{fix-cm}
\documentclass[twocolumn]{svjour3}          % twocolumn
\smartqed  % flush right qed marks, e.g. at end of proof
\usepackage{graphicx}
\usepackage{times}
\usepackage{amsmath}
\usepackage{amssymb}
\usepackage{breqn}
\usepackage{algorithmic}
\usepackage{float} % this is to better control the position of figures

\begin{document}

\title{Identifiability and Transportability in Dynamic Causal Networks%\thanks{Grants or other notes
}


\author{        Gilles Blondel \and
	Marta Arias         \and
        Ricard Gavald\`a
}


\institute{
           G. Blondel \at
              Universitat Polit\`ecnica de Catalunya \\
              \email{gillesblondel@yahoo.com}           %  \\
              \and
	M. Arias \at
              Universitat Polit\`ecnica de Catalunya \\
              \email{marias@cs.upc.edu}           %  \\
           \and
           R. Gavald\`a \at
              Universitat Polit\`ecnica de Catalunya \\
              \email{gavalda@cs.upc.edu}           %  \\
}

\date{}

\maketitle

\begin{abstract}
In this paper we propose a causal analog to the purely observational Dynamic Bayesian Networks, which we call Dynamic Causal Networks. We provide a sound and complete algorithm for identification of Dynamic Causal Networks, namely, for computing the effect of an intervention or experiment, based on passive observations only, whenever possible. We note the existence of two types of confounder variables that affect in substantially different ways the identification procedures, a distinction with no analog in either Dynamic Bayesian Networks or standard causal graphs. We further propose a procedure for the transportability of causal effects in Dynamic Causal Network settings, where the result of causal experiments in a source domain may be used for the identification of causal effects in a target domain. 
\end{abstract}

\section{Introduction}

Bayesian Networks (BN) are a canonical formalism for representing probability distributions over sets of variables and reasoning about them. 
A useful extension for modeling phenomena with recurrent temporal behavior are Dynamic Bayesian Networks (DBN). 
While regular BN are directed acyclic graphs, DBN may contain cycles, with some edges indicating dependence of a variable at time $t+1$ on another variable at time $t$. The cyclic graph in fact compactly represents an infinite acyclic graph formed by 
infinitely many replicas of the cyclic net, with some of the edges linking nodes in the same replica,
and others linking nodes in consecutive replicas. 

BN and DBN model conditional (in)dependences, so they are restricted to observational, non-interventional data or, equivalently, model association, not causality. Pearl's causal graphical models and do-calculus~\cite{pearl1994probabilistic} are a leading approach to modeling causal relations. They are formally similar to BN, as they are directed acyclic graphs with variables as nodes, but edges represent causality. A new notion is that of a {\em confounder,} an unobserved variable $X$ that causally influences two variables $Y$ and $Z$ so that the association between $Y$ and $Z$ may erroneously be taken for causal influence. Confounders are unnecessary in BNs since the association between $Y$ and $Z$ represents their correlation, with no causality implied. Causal graphical models allow to consider the effect of interventions or experiments, that is, externally forcing the values of some variables regardless of the variables that causally affect them, and studying the results. 

The do-calculus is an algebraic framework for reasoning about such experiments: An expression $\Pr(Y|do(X))$ indicates the probability distribution of a set of variables $Y$ upon performing an experiment on another set $X$. In some cases, the effect of such an experiment can be obtained from observational data only; this is convenient as some experiments may be impossible, expensive, or unethical to perform. 
When the expression $\Pr(Y|do(X))$, for a given causal network, can be rewritten as an expression containing only observational probabilities, without a do operator, we say that it is {\em identifiable}. \cite{shpitser2006identification,huang2006identifiability} showed that a do-expression is identifiable if and only if it can be rewritten in this way with a finite number of applications of the three rules of do-calculus, and \cite{shpitser2006identification} proposed the ID algorithm which performs this transformation if at all possible, or else returns {\em fail} indicating non-identifiability. 

In this paper we use a causal analog of DBNs to model phenomena where a finite set of variables evolves over time, with some variables causally influencing others at the same time $t$ but also others at time $t+1$. The infinite DAG representing these causal relations can be folded, if regular enough, into a directed graph, with some edges indicating intra-replica causal effects and other indicating  effect on variables in the next replica. Central to this representation is of course the intuitive fact that causal relations are directed towards the future, and never towards the past.

Existing work on dynamic causal models focuses on the discovery of causal models from data and on causal reasoning given a causal model. Regarding the discovery of causal models in dynamic systems \cite{iwasaki1989causality} and \cite{dash2008note} propose an algorithm to establish an ordering of the variables corresponding to the temporal order of propagation of causal effects. Methods for the discovery of cyclic causal graphs from data have been proposed using independent component analysis \cite{lacerda2012discovering} and using local d-separation criteria \cite{meek2014toward}. Existing algorithms for causal discovery from static data have been extended to the dynamic setting  by \cite{moneta2006graphical} and \cite{chicharro2015algorithms}. \cite{dahlhaus2003causality,white2010granger,white2011linking} discuss the discovery of causal graphs from time series by including granger causality concepts into their causal models. Our paper does not address causal discovery from data. Given the formal description of a dynamic system under a set of assumptions, our paper proposes algorithms that identify the modified trajectory of the system over time, after an intervention.

Dynamic causal systems are often modeled with sets of differential equations. However \cite{dashfundamental} \cite{dash2001caveats} \cite{dash2005restructuring} show the caveats of causal discovery of dynamic models based on differential equations which pass through equilibrium states, and how causal reasoning based on such models may fail. \cite{voortman2012learning} propose an algorithm for discovery of causal relations based on differential equations while ensuring those caveats due to system equilibrium states are taken into account. Time scale and sampling rate at which we observe a dynamic system play a crucial role in how well the obtained data may represent the causal relations in the system. \cite{aalen2014can} discuss the difficulties of representing a dynamic system with a DAG built from discrete observations and \cite{gong2015discovering} argue that under some conditions the discovery of temporal causal relations is feasible from data sampled at lower rate than the system dynamics. Our paper assumes that the observation time-scale is sufficiently small compared to the system dynamics, and that causal models include the non-equilibrium causal relations and not only those under equilibrium states. We assume that a stable set of causal dependencies exist which generate the system evolution along time. Our proposed algorithms take such models (and under these assumptions) as an input and predict the system evolution upon intervention on the system.

Regarding causal reasoning given a dynamic causal model, one line of research is based on time series and granger causality concepts \cite{eichler2010granger,eichler2012causal,eichler2012causal2}. \cite{queen2009intervention} use multivariate time series for identification of causal effects in traffic flow models. \cite{lauritzen2002chain} discuss intervention in dynamic systems in equilibrium, for several types of time-discreet and time-continuous generating processes with feedback. \cite{didelezcausal} uses local independence graphs to represent time-continuous dynamic systems and identify the effect of interventions by re-weighting involved processes.

Existing work on causal models does not thoroughly address causal reasoning in dynamic systems using do-calculus. \cite{eichler2010granger,eichler2012causal,eichler2012causal2} discuss back-door and front-door criteria in time-series but do not extend to the full power of do-calculus as a complete logic for causal identification. One of the advantages of do-calculus is its non-parametric approach so that it leaves the type of functional relation between variables undefined. Our paper extends the use of do-calculus to time series while requiring less restrictions than parametric causal analysis. Parametric approaches may require to differentiate the intervention impacts depending on the system state, non-equilibrium or equilibrium, while our non parametric approach is generic across system states.

%Similarly there is no requirement that the causal model be markovian of any order, as we impose no restrictions on the joint probability distribution of sampled data taken as input. 

Required work is to precisely define the notion and semantics of do-calculus and unobserved confounders in the dynamic setting and investigate whether and how existing do-calculus algorithms for identifiability of causal effects can be applied to the dynamic case. 

As a running example (more for motivation than for its accurate modeling of reality), let us consider two roads joining the same two cities, where drivers choose every day to use one or the other road. The average travel delay between the two cities any given day depends on the traffic distribution among the two roads. Drivers choose between a road or another depending on recent experience, in particular how congested a road was last time they used it. Figure~\ref{fig:dynamic_no_confounder_extended} indicates these relations: the weather($w$) has an effect on traffic conditions on a given day ($tr1$, $tr2$) which affects the travel delay on that same day ($d$). Driver experience influences the road choice next day, impacting $tr1$ and $tr2$. To simplify, we assume that drivers have short memory, being influenced by the conditions on the previous day only. This infinite network can be folded into a finite representation as shown in Figure~\ref{fig:dynamic_no_confounder_compact}, where $+1$ indicates an edge linking two consecutive replicas of the DAG. Additionally, if one assumes the weather to be an unobserved variable then it becomes a {\em confounder} as it causally affects two observed variables, as shown in Figure~\ref{fig:dcn_confounder_compact}. We call the confounders with causal effect over variables in the same time slice {\em static confounders}, and confounders with causal effect over variables at different time slices {\em dynamic confounders}. Our models allow for causal identification with both types of confounders, as will be discussed in Section~\ref{sec:id_dcn}. 

This setting enables the resolution of causal effect identification problems where causal relations are recurrent over time. These problems are not solvable in the context of classic DBNs, as causal interventions are not defined in such models. For this we use causal networks and do-calculus. However, time dependencies can't be modeled with static causal networks. As we want to predict the trajectory of the system over time after an intervention, we must use a dynamic causal network. Using our example, in order to reduce travel delay traffic controllers could consider actions such as limiting the number of vehicles admitted to one of the two roads. We would like to predict the effect of such action on the travel delay a few days later, e.g. $\Pr(d_{t+\alpha}|do(tr1_t))$.

% Previous work \cite{queen2009intervention} makes use of multivariate time series for identification of causal effects in traffic flow models. These models consider the evolution of sets of variables over time, and the effects of intervention from one set to another set. To perform causal reasoning in time series \cite{dahlhaus2003causality,white2010granger} use Granger's causality concepts along with DAGs. 

Our contributions in this paper are:

\begin{itemize}
\item We introduce Dynamic Causal Networks (DCN) as an analog of Dynamic Bayesian Networks for causal reasoning in domains that evolve over time. We show how to transfer the machinery of Pearl's do-calculus~\cite{pearl1994probabilistic} to DCN. 
\item We extend causal identification algorithms \cite{tianphd,shpitser2006identification,shpitser2012efficient} to the identifiability of causal effects in DCN settings. Given the expression $P(Y_{t+\alpha}|do(X_t))$, the algorithms either compute an equivalent do-free formula or conclude that such a formula does not exist. In the first case, the new formula provides the distribution of variables $Y$ at time $t+\alpha$
given that a certain experiment was performed on variables $X$ at time $t$. For clarity, we present first an algorithm that is sound but not complete (Section \ref{sec:id_dcn}), then give a complete one that is more involved to describe and justify (Section \ref{sec:complete_dcn}).
\item Unobserved confounder variables are central to the formalism of do-calculus. We observe a subtle difference between two types of unobserved confounder variables in DCN (which we call static and dynamic). This distinction is genuinely new to DCN, as it appears neither in DBN nor in standard causal graphs, yet the presence or absence of unobserved dynamic confounders has crucial impacts on the post-intervention evolution of the system over time and on the computational cost of the algorithms.
\item Finally, we extend from standard Causal Graphs to DCN the results by \cite{pearl2011transportability} on transportability, namely on whether causal effects obtained from experiments in one domain can be transferred to another domain with similar causal structure. This opens the way to studying relational knowledge transfer learning \cite{pan2010survey} of causal information in domains with a time component. 
\end{itemize}

\begin{figure}
\begin{center}
\includegraphics[width=0.45\textwidth]{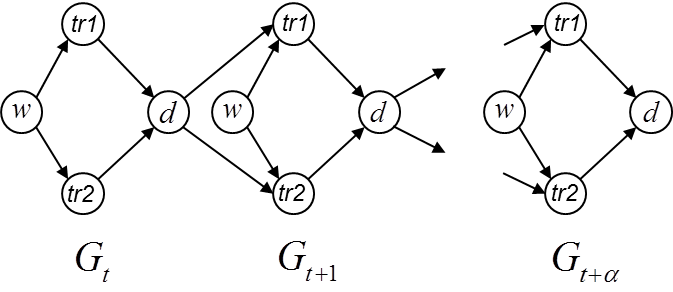}
\end{center}
\caption{A dynamic causal network. The weather $w$ has an effect on traffic flows $tr1$, $tr2$, which in turn have an impact on the average travel delay $d$. Based on the travel delay car drivers may choose a different road next time, having a causal effect on the traffic flows.}
\label{fig:dynamic_no_confounder_extended}
\end{figure}

\section{Previous Definitions and Results}

In this section we review the definitions and 
basic results on the three existing notions
that are the basis of our work: DBN, causal networks, and do-calculus. New definitions introduced in this paper
are left for Section~\ref{sec:definitions}. 

All formalisms in this paper model joint probability distributions over a set of variables. For static models (regular BN and Causal Networks) the set of variables is fixed. For dynamic models (DBN and DCN), there is a finite set of ``metavariables'', meaning variables that evolve over time. For a metavariable $X$ and an integer $t$, $X_t$ is the variable denoting the value of $X$ at time $t$. 

Let $V$ be the set of metavariables for a dynamic model. We say that a probability distribution $P$ is {\em time-invariant} if $P(V_{t+1}| V_t)$ is the same for every $t$. Note that this does not mean that $P(V_t) = P(V_{t+1})$ for every $t$, but rather that the laws governing the evolution of the variable do not change over time. For example, planets do change their positions around the Sun, but the Kepler-Newton laws that govern their movement do not change over time. Even if we performed an intervention (say, pushing the Earth away from the Sun for a while), these laws would immediately kick in again when we stopped pushing. The system would not be time-invariant if e.g. the gravitational constant changed over time.

\subsection{Dynamic Bayesian Networks}

Dynamic Bayesian Networks (DBN) are graphical models that generalize Bayesian Networks (BN) in order to model time-evolving phenomena. 
%\cite{daphnebook,ghahramani1998learning}
We rephrase them as follows. 

\begin{definition}
A DBN is a directed graph $D$ over a set of nodes that represent time-evolving metavariables. Some of the arcs in the graph have no label, and others are labeled ``$+1$''.
It is required that the sub-graph $G$ formed by the nodes and the unlabeled edges must be acyclic, therefore forming a Directed Acyclic Graph (DAG).
%(there may be two edges between two nodes, one unlabelled and another labelled $+1$). La phrase anterior no ho descarta aixo?
Unlabeled arcs denote dependence relations between metavariables within the same time step, and arcs labeled ``$+1$'' denote dependence between a variable at one time and another variable at the next time step.
\end{definition}

\begin{definition}
A DBN with graph  $G$ {\em represents} an infinite Bayesian Network $\hat G$ as follows. Timestamps $t$ are the integer numbers; $\hat G$ will thus be a biinfinite graph. For each metavariable $X$ in $G$ and each time step $t$ there is a variable 
$X_t$ in $\hat G$. The set of variables indexed by the same $t$ is denoted $G_t$ and called ``the slice at time $t$''. There is an edge from $X_t$ to $Y_t$ iff there is an unlabeled edge from $X$ to $Y$ in $G$, and there is an edge from $X_t$ to $Y_{t+1}$ iff 
there is an edge labeled ``$+1$'' from $X$ to $Y$ in $G$. Note that $\hat G$ is acyclic.
\end{definition}

The set of metavariables in $G$ is denoted $V(G)$, or simply $V$ when $G$ is clear from the context.
Similarly $V_t(G)$ or $V_t$ denote the variables in the $t$-th slice of $G$. 

%We will define as Dynamic Causal Network a DBN where the edges in $G_t$ and from $G_t$ to $G_{t+1}$ represent causal relations between the variables.

In this paper we will also use transition matrices to model probability distributions. Rows and columns are indexed by tuples assigning values to each variable, and the $(v,w)$ entry of the matrix represents the probability $P(V_{t+1} = w|V_t = v)$.
Let $T_t$ denote this transition matrix. Then we have, in matrix notation, $P(V_{t+1})=T_t\,P(V_t)$  and, more in general, $P(V_{t+\alpha}) = (\prod_{i=t}^{t+\alpha-1}T_i) \, P(V_t)$. In the case of time-invariant distributions, all $T_t$ matrices are the same matrix $T$, so  $P(V_{t+\alpha}) = T^{\alpha} P(V_t)$.

\subsection{Causality and Do-Calculus}

The notation used in our paper is based on causal models and do-calculus \cite{pearl1994probabilistic,pearl2000causality}.

\begin{definition}[Causal Model]
\label{def:causalnetwork}
A causal model over a set of variables $V$ is a tuple $M=\langle V,U,F,P(U) \rangle$, where U is a set of random variables that are determined outside the model ("exogenous" or "unobserved" variables) but that can influence the rest of the model, $V=\{ V_1,V_2,...V_n\}$ is a set of n variables that are determined by the model ("endogenous" or "observed" variables), $F$ is a set of n functions such that $V_{k} = f_k(pa(V_{k}),U_{k}, \theta_k)$, $pa(V_{k})$ are the parents of $V_{k}$ in $M$, $\theta_k$ are  a set of constant parameters and $P(U)$ is a joint probability distribution over the variables in $U$.
\end{definition}

In a causal model the value of each variable $V_k$ is assigned by a function $f_k$ which is determined by constant parameters $\theta_k$, a subset of $V$ called the "parents" of $V_k$ ($pa(V_{k}$)) and a subset of $U$ ($U_k$).

A causal model has an associated graphical representation (also called the "induced graph of the causal model") in which each observed variable $V_k$ corresponds to a vertex, there is one edge pointing to $V_k$ from each of its parents, i.e. from the set of vertex ${pa(V_{k})}$ and there is a doubly-pointed edge between the vertex influenced by a common unobserved variable in $U$ (see Figure~\ref{fig:dcn_confounder_compact}). In this paper we call the unobserved variables in $U$ "unobserved confounders" or "confounders" for simplicity. 

Causal graphs encode the causal relations between variables in a model. The primary purpose of causal graphs is to help estimate the joint probability of some of the variables in the model upon controlling some other variables by forcing them to specific values; this is called an action,  experiment or intervention. Graphically this is represented by removing all the incoming edges (which represent the causes) of the variables in the graph that we control in the experiment. Mathematically the $do()$ operator represents this experiment on the variables. Given a causal graph where $X$ and $Y$ are sets of variables, the expression $P(Y|do(X))$ is the joint probability of $Y$ upon doing an experiment on the controlled set $X$.

\begin{figure}
\begin{center}
\includegraphics[width=0.15\textwidth]{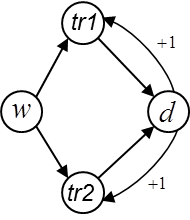}
\end{center}
\caption{Compact representation of a dynamic causal network where $+1$ indicates an edge linking a variable in $G_t$ with a variable in $G_{t+1}$.}
\label{fig:dynamic_no_confounder_compact}
\end{figure}

A causal relation represented by $P(Y|do(X))$ is said to be \textit{identifiable} if it can be uniquely computed from an observed, non-interventional, distribution of the variables in the model. In many real world scenarios it is impossible, impractical, unethical or too expensive to perform an experiment, thus the interest in evaluating its effects without actually having to perform the experiment.

The three rules of do-calculus \cite{pearl1994probabilistic} allow us to transform expressions with $do()$ operators into other equivalent expressions, based on the causal relations present in the causal graph. 

For any disjoint sets of variables $X$, $Y$, $Z$ and $W$:

\begin{enumerate}
\item
$P(Y|Z,W,do(X))=P(Y|W,do(X))$ if $(Y\perp Z|X,W)_{G_{\overline{X}}}$
\item 
$P(Y|W,do(X),do(Z))=P(Y|Z,W,do(X))$ if $(Y\perp Z|X,W)_{G_{\overline{X}\underline{Z}}}$
\item 
$P(Y|W,do(X),do(Z))=P(Y|W,do(X))$ if $(Y\perp Z|X,W)_{G_{\overline{X}\overline{Z(W)}}}$
\end{enumerate}

$G_{\overline{X}}$ is the graph $G$ where all edges incoming to $X$ are removed. $G_{\underline{Z}}$ is the graph $G$ where all edges outgoing from $Z$ are removed. Z(W) is the set of Z-nodes that are not ancestors of any W-nodes in $G_{\overline{X}}$.

Do-calculus was proven to be complete \cite{shpitser2006identification,huang2006identifiability} in the sense that if an expression
cannot be converted into a do-free one by iterative application of the three do-calculus rules, then it is not identifiable.

\subsection{The ID Algorithm}
\label{sec:id}

The ID algorithm \cite{shpitser2006identification}, and earlier versions by \cite{tianpearl2002,tian2004identifying} implement an  iterative application of do-calculus rules to transform a causal expression $P(Y|do(X))$ into an equivalent expression without any $do()$ terms in semi-Markovian causal graphs (with confounders). This enables the identification of interventional distributions from non-interventional data in such graphs.

The ID algorithm is sound and complete \cite{shpitser2006identification} in the sense that if a do-free equivalent expression exists it will be found by the algorithm, and if it does not exist the algorithm will exit and provide an error.

The algorithm specifications are as follows. Inputs: causal graph $G$, variable sets $X$ and $Y$, and a probability distribution $P$ over the observed variables in $G$; Output: an expression for $P(Y|do(X))$ without any $do()$ terms, or {\em fail}.

\textbf{\em Remark:\ } 
In our algorithms of Sections~\ref{sec:id_dcn} and \ref{sec:complete_dcn}, we may invoke the ID algorithm with a slightly more complex input: $P(Y|Z,do(X))$ (note the ``extra'' $Z$ to the right of the conditioning bar). In this case, we can solve the identification problem for the more complex expression with two calls to the ID algorithm using the following identity (definition of conditional probability):

\[P(Y|Z,do(X)) = \frac{P(Y,Z|do(X))}{P(Z|do(X))}\]

The expression $P(Y|Z,do(X))$ is thus identifiable if and only if both $P(Y,Z|do(X))$ and $P(Z|do(X))$ are \cite{shpitser2006identification}.

Another algorithm for the identification of causal effects is given in \cite{shpitser2012efficient}. 

The algorithms we propose in this paper show how to apply existing causal identification algorithms to the dynamic setting. In this paper we will refer as "ID algorithm" any existing causal identification algorithm.

\section{Dynamic Causal Networks and Do-Calculus}
\label{sec:definitions}
In this section we introduce the main definitions of this paper and state several lemmas based on the application of do-calculus rules to DCNs.

In the Definition~\ref{def:causalnetwork} of causal model the functions $f_k$ are left unspecified and can take any suitable form that best describes the causal dependencies between variables in the model. In natural phenomenon some variables may be time independent while others may evolve over time. However rarely does Pearl specifically treat the case of dynamic variables.

The definition of Dynamic Causal Network is an extension of Pearl's causal model in Definition \ref{def:causalnetwork}, by specifying that the variables are sampled over time, as in \cite{valdes2011effective}.

\begin{definition}[Dynamic Causal Network] 
\label{def:dynamiccausalnetwork}
A dynamic causal network $D$ is a causal model in which the set $F$ of functions is such that $V_{k,t} = f_k(pa(V_{k,t}),U_{k,t-\alpha}, \theta_k)$; where $V_{k,t}$ is the variable associated with the time sampling $t$ of the observed process $V_k$; $U_{k,t-\alpha}$ is the variable associated with the time sampling $t-\alpha$ of the unobserved process $U_k$; $t$ and $\alpha$ are discreet values of time.
\end{definition}

Note that $pa(V_{k,t})$ may include variables in any time sampling previous to $t$ up to and including $t$, depending on the delays of the direct causal dependencies between processes in comparison with the sampling rate. $U_{k,t-\alpha}$ may be generated by a noise process or by an unobserved confounder. In the case of noise, we assume that all noise processes $U_{k}$ are independent of each other, and that their influence to the observed variables happens without delay, so that $\alpha=0$. In the case of unobserved confounders, we assume $\alpha\geq 0$ as causes precede their effects.

To represent unobserved confounders in DCN, we extend to the dynamic context the framework developed in \cite{pearl1991theory} on causal model equivalence and latent structure projections. Let's consider the projection algorithm \cite{verma1993graphical}, which takes a causal model with unobserved variables and finds an equivalent model (with the same set of causal dependencies), called a "dependency-equivalent projection", but with no links between unobserved variables and where every unobserved variable is a parent of exactly two observed variables. 

The projection algorithm in DCN works as follows. For each pair $(V_{m},V_{n})$ of of observed processes, if there is a directed path from $V_{m,t}$ to $V_{n,t+\alpha}$ through unobserved processes then we assign a directed edge from $V_{m,t}$ to $V_{n,t+\alpha}$; however if there is a divergent path between them through unobserved processes then we assign a bidirected edge, representing an unobserved confounder. 

In this paper we represent all DCN by their dependency-equivalent projection. Also we assume the sampling rate to be adjusted to the dynamics of the observed processes. However, both the directed edges and the unobserved confounder paths may be crossing several time steps depending on the delay of the direct causal dependencies in comparison with the sampling rate. We now introduce the concept of static and dynamic confounder.

\begin{definition}[Static Confounder]
\label{def:staticconfounder}
Let $D$ be a DCN. Let $\beta$ be the maximal number of time steps crossed by any of the directed edges in $D$. Let $\alpha$ be the maximal number of time steps crossed by an unobserved confounder path. If $\alpha\leq\beta$ then the unobserved confounder is called Static.
%Let $D$ be a DCN. Let $V_{m,t} = f_k(pa(V_{m,t}),U_{m,t'}, \theta_m)$ and $V_{n,t+\alpha} = f_k(pa(V_{n,t+\alpha}),U_{n,t''}, \theta_n)$ be the time sampling $t$ and $t+\alpha$ of observed processes $V_m$ and $V_n$ respectively. If $U_m=U_n$, $t'=t''=t$ and $\alpha=0$ then $U_{m,t'}=U_{n,t''}$ is called an Unobserved Static Confounder, or Static Confounder for short.
\end{definition}

\begin{definition}[Dynamic Confounder]
\label{def:dynamicconfounder}
Let $D$, $\beta$ and $\alpha$ be as in Definition \ref{def:staticconfounder}. If $\alpha>\beta$ then the unobserved confounder is called Dynamic. More specifically, if $\beta<\alpha\leq 2\beta$ we call it "first order" Dynamic Confounder; if $\alpha>2\beta$ we call it "higher order" Dynamic Confounder.
%Let $D$, $V_{m,t}$ and $V_{n,t+\alpha}$ as in Definition~\ref{def:staticconfounder}. If $U_m=U_n$, $t'=t''$ and $\alpha\neq 0$ then $U_{m,t'}=U_{n,t''}$ is called an Unobserved Dynamic Confounder, or Dynamic Confounder for short.
\end{definition}

In this paper, we consider three case scenarios in regards to DCN and their time-invariance properties. If a DCN $D$ contains only static confounders we can construct a first order Markov process in discrete time, by taking $\beta$ (per Definition \ref{def:staticconfounder}) consecutive time samples of the observed processes $V_k$ in $D$. This does not mean the DCN generating functions $f_k$ in Definition \ref{def:dynamiccausalnetwork} are time-invariant, but that a first order Markov chain can be built over the observed variables when marginalizing the static confounders over $\beta$ time samples.

In a second scenario, we consider DCN with first order dynamic confounders. We can still construct a first order Markov process in discrete time, by taking $\beta$ consecutive time samples. However we will see in later sections how the effect of interventions on this type of DCN has a different impact than on DCN with static confounders.

Finally, we consider DCN with higher order dynamic confounders, in which case we may construct a first order Markov process in discrete time by taking a multiple of $\beta$ consecutive time samples.

As we will see in later sections, the difference between these three types of DCN is crucial in the context of identifiability. 
Dynamic confounders cause a time invariant transition matrix to become dynamic after an intervention, e.g. the post-intervention transition matrix will change over time. However, if we perform an intervention on a DCN with static confounders, the network will return to its previous time-invariant behavior after a transient period. These differences have a great impact on the complexity of the causal identification algorithms that we present. 

Considering that causes precede their effects, the associated graphical representation of a DCN is a DAG. All DCN can be represented as a biinfinite DAG with vertices $V_{k,t}$; edges from $pa(V_{k,t})$ to $V_{k,t}$; and confounders (bi-directed edges). DCN with static confounders and DCN with first order dynamic confounders can be compactly represented as $\beta$ time samples (a multiple of $\beta$ time samples for higher order dynamic confounders) of the observed processes $V_{k,t}$; their corresponding edges and confounders; and some of the directed and bi-directed edges marked with a "+1" label representing the dependencies with the next time slice of the DCN.

\begin{definition}[Dynamic Causal Network identification] Let $D$ be a DCN, and $t$, $t+\alpha$ be two time slices of $D$. Let $X$ be a subset of $V_{t}$ and $Y$ be a subset of $V_{t+\alpha}$. 
The DCN identification problem consists of computing the probability distribution $P(Y|do(X))$ from the observed probability distributions in $D$, i.e. computing an expression for the distribution containing no do() operators.
\end{definition}

In the definition above we always assume that $X$ and $Y$ are disjoint. % remove this paragraph for short version
In this version we only consider the case in which all intervened variables $X$ are in the same time sample. It is not difficult to extend our algorithm to the general case.

The following lemma is based on the application of do-calculus to DCN. Intuitively, future actions have no impact on the past.

\begin{lemma}[Future actions]\label{lem:futureact} 
Let $D$ be a DCN. 
Take any sets $X \subseteq V_{t}$ and $Y \subseteq V_{t-\alpha}$, 
with $\alpha >0$. 
Then for any set $Z$ the following equalities hold:

\begin{enumerate}
\item $P(Y|do(X),do(Z))=P(Y|do(Z))$
\item $P(Y|do(X))=P(Y)$
\item $P(Y|Z,do(X))=P(Y|Z)$ whenever $Z \subseteq V_{t-\beta}$ with $\beta > 0$. 
\end{enumerate}
\end{lemma}

\begin{proof}
The first equality 
derives from rule 3 and the proof in~\cite{shpitser2006identification} that interventions on variables which are not ancestors of $Y$ in $D$ have no effect on $Y$. The second is the special case $Z=\emptyset$. We can transform the third expression using the equivalence $P(Y|Z,do(X))= P(Y,Z|do(X))/P(Z|do(X))$; since $Y$ and $Z$ precede $X$ in $D$, by rule 3
$P(Y,Z|do(X)) = P(Y,Z)$ and 
$P(Z|do(X)) = P(Z)$,
and then the above equals
$P(Y,Z)/P(Z) = P(Y|Z)$.
\qed\end{proof}

In words, traffic control mechanisms applied next week have no causal effect on the traffic flow this week.

% lowercase lemma because it's not followed by a number. 
The following lemma limits the size of the graph to be used for the identification of DCNs.

\begin{lemma}
\label{lem:graphID}
Let $D$ be a DCN.
%We would like to identify $P(Y|do(X))$. 
Let $G$ be the sub-graph of $\hat D$ consisting of all time slices in between (and including) $t_x$ and $t_y$. Let $G'$ be the sub-graph $G$ augmented with the time slice preceding it. If $P(Y|do(X))$ is identifiable in $D$ then it is identifiable in $G'$ and the identification provides the same result on both graphs.
\end{lemma}
\begin{proof} (sketch) By C-component factorization \cite{tianphd}, we decompose the problem as that of  identification of each C-component in $D$ and (if all C-components are identifiable) multiplying all identified quantities to obtain $P(Y|do(X))$. C-components are sets of variables linked by confounder edges in the graph $D\setminus X$. An identifiable C-component is computed as the product of $P(v_i|V_{\pi}^{i-1})$ for each variable $v_i$ in the C-component, where $V_{\pi}^{i-1}$ is the set of all variables preceding $v_i$ in some topological ordering $\pi$ \cite{shpitser2006identification,tianphd}. The C-component factorization involving all the variables preceding the set $G$ leads to the joint distribution of these variables, and can be computed using the joint distribution of the time slice preceding $G$ alone. Also, non-ancestors of $Y$ can be ignored from the graph, by application of do-calculus rule 3, so time slices succeeding $G$ can be discarded. Therefore the identification problem can be computed in the limited graph $G'$.

This result is crucial to reduce the complexity of identification algorithms in dynamic settings. In order to describe the evolution of a dynamic system over time, after an intervention, %with existing identification algorithm \cite{shpitser2006identification,shpitser2012efficient}, 
we can run a causal identification algorithm over a limited number of time slices of the DCN, instead of the entire DCN.
\qed\end{proof}

\section{Identifiability in Dynamic Causal Networks}
\label{sec:id_dcn}

In this section we analyze the identifiability of causal effects in the DCN setting. We first study DCNs with static confounders and propose a method for identification of causal effects in DCNs using transition matrices. Then we extend the analysis and identification method to DCNs with dynamic confounders. As discussed in Section \ref{sec:definitions}, both the DCNs with static confounders and  with dynamic confounders can be represented as a Markov chain. For graphical and notational simplicity, we represent these DCN graphically as recurrent time slices as opposed to the shorter time samples, on the basis that one time slice contains as many time samples as the maximal delay of direct causal influence among the processes. Also for notational simplicity we assume the transition matrix from one time slice to the next to be time-invariant; however removing this restriction would not make any of the lemmas, theorems or algorithms invalid, as they are the result of graphical non-parametric reasoning.

Consider a DCN under the above assumptions, and let $T$ be its time invariant transition matrix from any time slice $V_{t}$ to $V_{t+1}$. We assume that there is some time $t_0$ such that the distribution $P(V_{t_0})$ is known. Fix now $t_x > t_0$ and a set $X \subseteq V_{t_x}$. 
We will now see how performing an intervention on $X$ affects the distributions in $D$.

We begin by stating a series of lemmas that apply to DCNs in general.

\begin{lemma}
\label{lem:Tbefore}
Let $t$ be such that $t_0 \le t < t_x$, with $X \subseteq V_{t_x}$. Then $P(V_{t}|do(X)) = T^{t-t_0} P(V_{t_0})$.
Namely, transition probabilities
are not affected by an intervention in the future. 
\end{lemma}
\begin{proof}
By Lemma~\ref{lem:futureact}, (2),
$P(V_{t}|do(X))= P(V_{t})$ for all such $t$. 
By definition of $T$, this equals $T\,P(V_{t-1})$. Then induct on $t$ with $P(V_{t_0}) = T^0 P(V_{t_0})$ as base.
\qed\end{proof}

\begin{lemma}
\label{lem:A}
Assume that an expression $P(V_{t+\alpha}|V_{t},do(X))$ is identifiable for some $\alpha>0$. Let $A$ be the matrix whose entries $A_{ij}$ correspond to the probabilities $P(V_{t+\alpha} = v_j|V_t = v_i, do(X))$. Then $P(V_{t+\alpha}|do(X)) = A\,P(V_t|do(X))$.
\end{lemma}
\begin{proof} 
Case by case evaluation of $A$'s entries. 
\qed\end{proof}

\subsection{DCNs with Static Confounders}
\label{sec:static}

Static confounders impact sets of variables within one time slice only, and there are no confounders between variables at different time slices (see Figure~\ref{fig:dcn_confounder_compact}).

The following three lemmas are based on the application of do-calculus to DCNs with static confounders. Intuitively, conditioning on the variables that cause time dependent effects d-separates entire parts (future from past) of the DCN (Lemmas \ref{lem:past}, \ref{lem:futureobs}, \ref{lem:Tafter}).

\begin{lemma}[Past observations and actions]\label{lem:past}
Let $D$ be a DCN with static confounders. Take any set $X$. Let $C \subseteq V_{t}$ be the set of variables in $G_t$ that are direct causes of variables in $G_{t+1}$. Let $Y \subseteq V_{t+\alpha}$ and $Z \subseteq V_{t-\beta}$, with $\alpha > 0$ and $\beta > 0$ (positive natural numbers). The following distributions are identical:

\begin{enumerate}
\item 
$P(Y | do(X),Z,C)$
\item 
$P(Y | do(X),do(Z),C)$
\item 
$P(Y | do(X),C)$
\end{enumerate}
\end{lemma}

\begin{proof}
By the graphical structure of a DCN with static confounders, conditioning on $C$ d-separates $Y$ from $Z$.
The three rules of do-calculus apply, and (1) equals (3) by rule 1, (1) equals (2) by rule 2, and also (2) equals (3) by rule 3.
\qed\end{proof}

In our example, we want to predict the traffic flow $Y$ in two days caused by traffic control mechanisms applied tomorrow $X$, and conditioned on the traffic delay today $C$. Any traffic controls $Z$ applied before today are irrelevant, because their impact is already accounted for in $C$.

\begin{lemma}[Future observations]\label{lem:futureobs}
Let $D$, $X$ and $C$ be as in Lemma~\ref{lem:past}. Let $Y \subseteq V_{t-\alpha}$ and $Z \subseteq V_{t+\beta}$, with $\alpha > 0$ and $\beta > 0$, then:
$$
P(Y|do(X),Z,C)=P(Y|do(X),C)
$$
\end{lemma}

\begin{proof}
By the graphical structure of a DCN with static confounders, conditioning on $C$ d-separates $Y$ from $Z$ and the expression is valid by rule 1 of do-calculus.
\qed\end{proof}

In our example, observing the travel delay today makes observing the future traffic flow irrelevant to evaluate yesterday''s traffic flow.

\begin{lemma}
\label{lem:Tafter}
Let $t > t_x$. Then 
$P(V_{t+1}|do(X)) = T\,P(V_{t}|do(X))$.
Namely, transition probabilities
are not affected by intervention more than one time unit 
in the past. 
\end{lemma}
\begin{proof}
$P(V_{t+1}|do(X)) = T'\,P(V_{t}|do(X))$ where the elements of $T'$ are $P(V_{t+1}|V_t, do(X))$. As $V_{t}$ includes all variables in $G_{t}$ that are direct causes of variables in $G_{t+1}$, conditioning on $V_{t}$ d-separates $X$ from $V_{t+1}$. By Lemma~\ref{lem:past} we exchange the action $do(X)$ by the observation $X$ and so $P(V_{t+1}|V_t, do(X)) = P(V_{t+1}|V_t, X)$. Moreover, $V_t$ d-separates $X$ from $V_{t+1}$, so they are statistically independent given $V_t$. Therefore, $P(V_{t+1}|V_t, do(X)) = P(V_{t+1}|V_t, X) = P(V_{t+1}|V_t)$ which are the elements of matrix $T$ as required.
\qed\end{proof}

\begin{theorem}
\label{thm:static}
Let $D$ be a DCN with static confounders, and transition matrix $T$.
Let $X\subseteq V_{t_x}$ and $Y\subseteq V_{t_y}$ for two time points $t_x < t_y$.
If the expression $P(V_{t_x+1}|V_{t_x-1}, do(X))$ is identifiable with corresponding transition matrix $A$, then $P(Y|do(X))$ is identifiable and 
$$P(Y|do(X))=\sum_{V_{t_y}\setminus Y} T^{t_y-(t_x+1)}AT^{t_x-1-t_0}P(V_{t_0}).$$
\end{theorem}

\begin{proof} 
Applying Lemma~\ref{lem:Tbefore}, we obtain that
$P(V_{t_x-1}|do(X)) = T^{t_x-1-t_0}P(V_{t_0})$.
We  assumed that $P(V_{t_x+1}|V_{t_x-1}, do(X))$ is identifiable, 
therefore Lemma~\ref{lem:A} guarantees that 
$$P(V_{t_x+1}|do(X)) = A\,P(V_{t_x-1}|do(X)) = A\,T^{t_x-1-t_0}P(V_{t_0}).$$
Finally, $P(V_{t_y}|do(X)) = T^{(t_y-(t_x+1))} P(V_{t_x+1}|do(X))$ by repeatedly applying Lemma~\ref{lem:Tafter}.
$P(Y|do(X))$ is obtained by marginalizing variables in $V_{t_y}\setminus Y$ in the resulting expression $T^{t_y-(t_x+1)}AT^{t_x-1-t_0}P(V_{t_0})$.
\qed\end{proof}

As a consequence of Theorem~\ref{thm:static}, causal identification of $D$ reduces to the problem of identifying
the expression $P(V_{t_x+1}|V_{t_x-1},do(X))$. 
The ID algorithm can be used to check whether this expression is identifiable and, if it is, compute its joint probability from observed data.

Note that Theorem~\ref{thm:static} holds without the assumption of transition matrix time-invariance by replacing powers of $T$ with products of matrices $T_t$.

\begin{figure}
\begin{center}
\includegraphics[width=2cm]{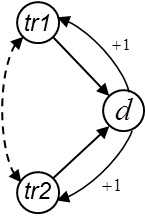}
\end{center}
\caption{Dynamic Causal Network where $tr1$ and $tr2$ have a common unobserved cause, a {\em confounder}. Since both variables are in the same time slice, we call it a {\em static confounder}.}

\label{fig:dcn_confounder_compact}
\end{figure}

\subsubsection{DCN-ID Algorithm for DCNs with Static Confounders}
\label{sec:static_algo}

The DCN-ID algorithm for DCNs with static confounders is given in Figure~\ref{fig:static-algo}. 
Its soundness is immediate from Theorem~\ref{thm:static}, the soundness of the ID algorithm \cite{shpitser2006identification}, and Lemma~\ref{lem:graphID}.

\begin{theorem}[Soundness] Whenever DCN-ID returns a distribution for  $P(Y|do(X))$, it is correct.\ \qed
\end{theorem}

Observe that line 2 of the algorithm calls ID with a graph of size $4|G|$. By the remark of Section~\ref{sec:id}, this means two calls but notice that in this case we can spare the call for the ``denominator'' $P(V_{t_x-1}|do(X))$ because Lemma~\ref{lem:futureact} guarantees $P(V_{t_x-1}|do(X)) = P(V_{t_x-1})$. Computing transition matrix A on line 3 has complexity $O((4k)^{(b+2)})$, where $k$ is the number of variables in one time slice and $b$ the number of bits encoding each variable.
The formula on line 4 is the multiplication of $P(V_{t_0})$ by $n=(t_y-t_0)$ matrices, which has complexity $O(n.b^2)$. To solve the same problem with the ID algorithm would require running it on the entire graph of size $n|G|$ and evaluating the resulting joint probability with complexity $O((n.k)^{(b+2)})$ compared to $O((4k)^{(b+2)}+n.b^2)$ with DCN-ID.

If the problem we want to solve is evaluating the trajectory of the system over time 
$$(P(V_{t_x+1}),P(V_{t_x+2}),P(V_{t_x+3}),...P(V_{t_x+n}))$$ 
after an intervention at time slice $t_x$, with ID we would need to run ID $n$ times and evaluate the $n$ outputs with overall complexity $O((k)^{(b+2)}+(2k)^{(b+2)}+(3k)^{(b+2)}+...+(n.k)^{(b+2)})$. Doing the same with DCN-ID requires running ID one time to identify $P(V_{t_x+1})$, evaluating the output and applying successive transition matrix multiplications to obtain the joint probability of the time slices thereafter, with resulting complexity $O((4k)^{(b+2)}+n.b^2)$.

\begin{figure}[ht]
\hrule\medskip
Function \textbf{DCN-ID}($Y$,$t_y$, $X$,$t_x$, $G$,$C$,$T$,$P(V_{t_0})$)

INPUT: 
\begin{itemize}
\item DCN defined by a causal graph $G$ 
on a set of variables $V$ and a set $C \subseteq V \times V$ describing causal relations from $V_t$ to $V_{t+1}$ for every $t$
\item transition matrix $T$ for $G$ derived
from observational data
\item a set $Y$ included in $V_{t_y}$
\item a set $X$ included in $V_{t_x}$
\item distribution $P(V_{t_0})$ at the initial state, 
\end{itemize}

OUTPUT: The distribution  $P(Y|do(X))$, or else FAIL

\begin{enumerate}
\item let $G'$ be the acyclic graph formed by joining $G_{t_x-2}$, $G_{t_x-1}$, $G_{t_x}$, and $G_{t_x+1}$
by the causal relations given by $C$;
\item run the standard ID algorithm for expression $P(V_{t_x+1}|V_{t_x-1},do(X))$ on $G'$; if it returns FAIL, return FAIL;
\item else, use the resulting distribution to compute the transition matrix $A$, where $A_{ij} = P(V_{t_x+1}=v_i|V_{t_x-1}=v_j, do(X))$;
\item return $\sum_{V_{t_y}\setminus Y} T^{t_y-(t_x+1)}\,A\,T^{t_x-1-t_0}\,P(V_{t_0})$;
\end{enumerate}

\caption{The DCN-ID algorithm for DCNs with static confounders}
\label{fig:static-algo}
\medskip\hrule
%\smallskip 
%\hrule
\end{figure}

\subsection{DCNs with Dynamic Confounders}
\label{sec:dynamic}
We now discuss the case of DCNs with dynamic confounders, that is, with confounders that influence variables in consecutive time slices.

The presence of dynamic confounders d-connects time slices, and we will see in the following lemmas how this may be an obstacle for the identifiability of the DCN.

In the presence of dynamic confounders, Lemma~\ref{lem:Tafter} does no longer hold since d-separation is no longer guaranteed. As a consequence, we cannot guarantee the DCN will recover its ``natural'' (non-interventional) transition probabilities from one cycle to the next after the intervention is performed.

Our statement of the identifiability theorem for DCNs with dynamic confounders is weaker and includes in its assumptions those conditions that can no longer be guaranteed.

\begin{theorem}
\label{thm:dynamic}
Let $D$ be a DCN with dynamic confounders. Let $T$ be its transition matrix under no interventions. We further assume that:
\begin{enumerate}
\item $P(V_{t_x+1}|V_{t_x-1}, do(X))$ is identifiable by matrix $A$ 
\item For all
$t > t_x+1$, $P(V_{t}|V_{t-1}, do(X))$ is identifiable by matrix $M_t$
\end{enumerate}
Then $P(Y|do(X)$ is identifiable and computed by
\[P(Y|do(X))=\sum_{V_{t_y}\setminus Y} \left[\prod\limits_{t=t_x+2}^{t_y} M_t\right]\,A\,T^{t_x-1-t_0}P(V_{t_0}).\]
\end{theorem}
\begin{proof} 
Similar to the proof of Theorem~\ref{thm:static}. By Lemma~\ref{lem:Tbefore}, we can compute the distribution up to time $t_x-1$ as 
$$P(V_{t_x-1}|do(X)) = T^{t_x-1-t_0}P(V_{t_0}).$$ 
Using the first assumption in the statement of the theorem, by Lemma~\ref{lem:A} we obtain
$$P(V_{t_x+1}|do(X)) = A\,T^{t_x-1-t_0}P(V_{t_0}).$$
Then, we compute the final $P(V_{t_y}|do(X))$ using the matrices $M_t$ from the statement of the theorem that allows us to compute probabilities for subsequent time-slices. Namely, 
\begin{align*}
P(V_{t_x+2}|do(X)) &= M_{t_x+2}\,A\,T^{t_x-1-t_0}P(V_{t_0}), \\
P(V_{t_x+3}|do(X)) &= M_{t_x+3}\,M_{t_x+2}\,A\,T^{t_x-1-t_0}P(V_{t_0}),
\end{align*}
and so on until we find
\[P(V_{t_y}|do(X)) = \left[\prod\limits_{t=t_x+2}^{t_y} M_t\right]\,A\,T^{t_x-1-t_0}P(V_{t_0}).\]
Finally, the do-free expression of $P(Y|do(X))$ is obtained by marginalization over variables of $V_{t_y}$ not in $Y$.
\qed\end{proof}

Again, note that Theorem~\ref{thm:dynamic} holds without the assumption of transition matrix time-invariance by replacing powers of $T$ with products of matrices $T_t$.

\subsubsection{DCN-ID Algorithm for DCNs with Dynamic Confounders}
\label{sec:dynamic_algo}
\begin{figure}[H]
\hrule\medskip
Function \textbf{DCN-ID}($Y$,$t_y$, $X$,$t_x$, $G$,$C$,$C'$,$T$,$P(V_{t_0})$)

INPUT: 
\begin{itemize}
\item DCN defined by a causal graph $G$ 
on a set of variables $V$ and a set $C \subseteq V \times V$ describing causal relations from $V_t$ to $V_{t+1}$ for every $t$, and a set $C' \subseteq V \times V$ describing confounder relations from $V_t$ to $V_{t+1}$ for every $t$
\item transition matrix $T$ for $G$ derived
from observational data
\item a set $Y$ included in $V_{t_y}$
\item a set $X$ included in $V_{t_x}$
\item distribution $P(V_{t_0})$ at the initial state, 
\end{itemize}

OUTPUT: The distribution  $P(Y|do(X))$, or else FAIL

\begin{enumerate}
\item let $G'$ be the acyclic graph formed by joining $G_{t_x-2}$, $G_{t_x-1}$, $G_{t_x}$, and $G_{t_x+1}$
by the causal relations given by $C$ and confounders given by $C'$;
\item run the standard ID algorithm for expression $P(V_{t_x+1}|V_{t_x-1},do(X))$ on $G'$; if it returns FAIL, return FAIL;
\item else, use the resulting distribution to compute the transition matrix $A$, where $A_{ij} = P(V_{t_x+1}=v_i|V_{t_x-1}=v_j, do(X))$;
\item for each $t$ from $t_x+2$ up to $t_y$:
	\begin{enumerate}
    \item let $G''$ be the causal graph composed of time slices $G_{t_x-1}$, $G_{t_x}$, \dots, $G_{t}$
	\item run the standard ID algorithm on $G''$ for the expression $P(V_t|V_{t-1},do(X))$; if it returns FAIL, return FAIL;
    \item else, use the resulting distribution to compute the transition matrix $M_t$, where $(M_t)_{ij} = P(V_{t}=v_i|V_{t-1}=v_j, do(X))$;
	\end{enumerate}
\item return $\sum_{V_{t_y}\setminus Y} \left[\prod\limits_{t=t_x+2}^{t_y} M_t\right]\,A\,T^{t_x-1-t_0}P(V_{t_0})$;
\end{enumerate}

\caption{The DCN-ID algorithm for DCNs with dynamic confounders}
\label{fig:dynamic-algo}
\medskip\hrule
%\smallskip 
%\hrule
\end{figure}

The DCN-ID algorithm for DCNs with dynamic confounders  is given in Figure~\ref{fig:dynamic-algo}.

Its soundness is immediate from Theorem~\ref{thm:dynamic}, the soundness of the ID algorithm \cite{shpitser2006identification}, and Lemma~\ref{lem:graphID}.

\begin{theorem}[Soundness] Whenever DCN-ID returns a distribution for  $P(Y|do(X))$, it is correct.\ \qed
\end{theorem}

Notice that this algorithm is more expensive than the DCN-ID algorithm for DCNs with static confounders. In particular, it requires $(t_y - t_x)$ calls to the ID algorithm with increasingly larger chunks of the DCN. To identify a single future effect $P(Y|do(X)$ it may be simpler to invoke Lemma~\ref{lem:graphID} and do a unique call to the ID algorithm for the expression $P(Y|do(X)$ restricted to the causal graph formed by time-slices $G_{t_x-1}$, ..., $G_{t_y}$. However, to predict the trajectory of the system over time after an intervention, the DCN-ID algorithm for dynamic confounders directly identifies the post-intervention transition matrix and its evolution. A system characterized by a time-invariant transition matrix before the intervention will be characterized by a time dependent transition matrix, given by the DCN-ID algorithm, after the intervention. This dynamic view offers opportunities for the analysis of the time evolution of the system, and conditions for convergence to a steady state.

\section{Complete DCN Identifiability}
\label{sec:complete_dcn}

In this section we show that the identification algorithms as formulated in previous sections are not complete, and we develop complete algorithms for complete identification of DCNs. To prove completeness we use previous results \cite{shpitser2006identification}. It is shown there that the absence of a structure called 'hedge' in the graph is a sufficient and necessary condition for identifiability. We first define some graphical structures that lead to the definition of hedge, in the context of DCNs.

\begin{definition}[C-component]
Let $D$ be a DCN. Any maximal subset of variables of $D$ connected by bidirected edges (confounders) is called a C-component.
\end{definition}

\begin{definition}[C-forest]
Let $D$ be a DCN and $C$ a C-component of $D$. If all variables in $C$ have at most one child, then $C$ is called a C-forest. The set $R$ of variables in $C$ that have no descendants is called the C-forest root, and the C-forest is called $R$-rooted.
\end{definition}

\begin{definition}[Hedge]
\label{def:hedge}
Let $X$ and $Y$ be sets of variables in $D$. Let $F$ and $F'$ be two $R$-rooted C-forests such that $F'\subseteq F$, $F\cap X \neq \emptyset$, $F'\cap X = \emptyset$, $R\subset An(Y)_{D_{\bar{X}}}$. Then $F$ and $F'$ form a Hedge for $P(Y|do(X))$ in $D$.
\end{definition}

Notice that $An(Y)_{D_{\bar{X}}}$ refers to those variables that are ancestors of $Y$ in the causal network $D$ where incoming edges to $X$ have been removed. We may drop the subscript as in $An(Y)$ in which case we are referring to the ancestors of $Y$ in the unmodified network $D$ (in which case, the network we refer to should be clear from the context).
Moreover, we overload the definition of the ancestor function and we use $An(Z,V)$ to refer to the ancestors of the \emph{union} of sets $Z$ and $V$, that is, $An(Z,V) = An(Z \cup V)$.

The presence of a hedge prevents the identifiability of causal graphs \cite{shpitser2006identification}. Also any non identifiable graph necessarily contains a hedge. These results applied to DCNs lead to the following lemma.

\begin{lemma}[DCN complete identification]
\label{lem:dcn_complete}
Let $D$ be a DCN with confounders. Let $X$ and $Y$ be sets of variables in $D$. $P(Y|do(X))$ is identifiable iif there is no hedge in $D$ for $P(Y|do(X))$.
\end{lemma}

We can show that the algorithms presented in the previous section, in some cases introduce hedges in the sub-networks they analyze, even if no hedges existed in the original expanded network.

\begin{lemma} The DCN-ID algorithms for DCNs with static confounders (Section~\ref{sec:static}) and dynamic confounders (Section~\ref{sec:dynamic}) are not complete.
\label{lem:dcn_id_not_complete}
\end{lemma}

\begin{proof}
Let $D$ be an DCN. Let $X$ be such that $D$ contains two $R$-rooted C-forests $F$ and $F'$, $F'\subseteq F$, $F\cap X \neq 0$, $F'\cap X = 0$. Let $Y$ be such that $R\not\subset An(Y)_{D_{\bar{X}}}$. The condition for $Y$ implies that $D$ does not contain a hedge, and is therefore identifiable by Lemma~\ref{lem:dcn_complete}. Let the set of variables at time slice $t_x+1$ of $D$,  $V_{t_x+1}$, be such that $R\subset An(V_{t_x+1})_{D_{\bar{X}}}$. By Definition~\ref{def:hedge}, $D$ contains a hedge for $P(V_{t_x+1}|V_{t_x-1},do(X))$. The identification of $P(Y|do(X))$ requires the DCN-ID algorithm to identify $P(V_{t_x+1}|V_{t_x-1},do(X))$ which fails.
\qed\end{proof}

\begin{figure}[H]
\begin{center}
\includegraphics[width=8cm,height=4cm]{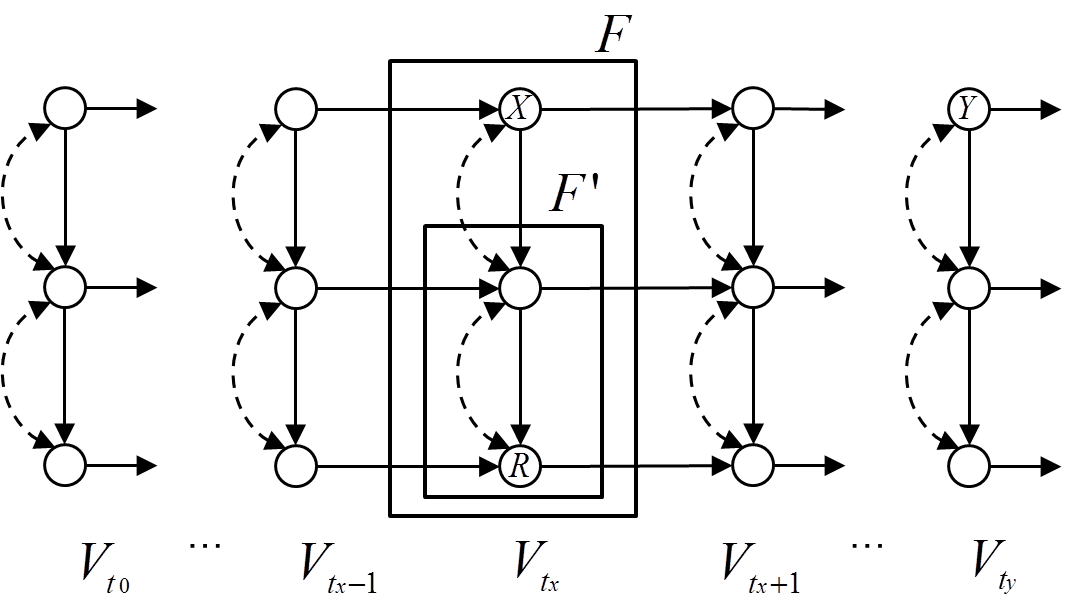}
\end{center}
\caption{Identifiable Dynamic Causal Network which the DCN-ID algorithm fails to identify. $F$ and $F'$ are $R$-rooted C-forests, but since $R$ is not an ancestor of $Y$ there is no hedge for $P(Y|do(X))$. However $R$ is an ancestor of $V_{t_x+1}$ and DCN-ID fails when finding the hedge for $P(V_{t_x+1}|V_{t_x-1}, do(X))$.}
\label{fig:dcn_no_hedge}
\end{figure}

Figure~\ref{fig:dcn_no_hedge} shows an identifiable DCN that DCN-ID fails to identify.

The proof of Lemma~\ref{lem:dcn_id_not_complete} provides the framework to build a complete algorithm for identification of DCNs.

\subsection{Complete DCN identification algorithm with Static Confounders}
\label{sec:completestatic}

The DCN-ID algorithm can be modified so that no hedges are introduced if none existed in the original network. This is done at the cost of more complicated notation, because the fragments of network to be analyzed do no longer correspond to natural time slices. More delicate surgery is needed.

\begin{lemma}
\label{lem:hedge_static}
Let $D$ be a DCN with static confounders. Let $X\subseteq V_{t_x}$ and $Y\subseteq V_{t_y}$ for two time slices $t_x < t_y$. If there is a hedge $H$ for $P(Y|do(X))$ in $D$ then $H\subseteq V_{t_x}$.
\end{lemma}

\begin{proof}
By definition of hedge, $F$ and $F'$ are connected by confounders to $X$. As $D$ has only static confounders $F$, $F'$ and $X$ must be within $t_x$.
\qed\end{proof}

\begin{lemma}
\label{lem:static_complete}
Let $D$ be a DCN with static confounders. Let $X\subseteq V_{t_x}$ and $Y\subseteq V_{t_y}$ for two time slices $t_x < t_y$.
$P(Y|do(X))$ is identifiable if and only if $P(V_{t_x+1}\cap An(Y)|V_{t_x-1}, do(X))$ is identifiable.
\end{lemma}

\begin{proof} 
(\textbf{if}) By Lemma~\ref{lem:dcn_complete}, if 
\begin{align*}
P(V_{t_x+1}\cap An(Y)&|V_{t_x-1},do(X)) \\
&=\frac{P(V_{t_x+1}\cap An(Y),V_{t_x-1}|do(X))}{P(V_{t_x-1})}
\end{align*} 
is identifiable then there is no hedge for this expression in $D$. By Lemma~\ref{lem:hedge_static} if $D$ has static confounders, a hedge must be within time slice $t_x$. If time slice $t_x$ does not contain two $R$-rooted C-forests $F$ and $F'$ such that $F'\subseteq F$, $F\cap X \neq 0$, $F'\cap X = 0$, then there is no hedge for any set $Y$ so there is no hedge for the expression $P(Y|do(X))$ which makes it identifiable. Now let's assume time slice $t_x$ contains two $R$-rooted C-forests $F$ and $F'$ such that $F'\subseteq F$, $F\cap X \neq 0$, $F'\cap X = 0$, then $R\not\subset An(V_{t_x+1}\cap An(Y),V_{t_x-1})_{D_{\bar{X}}}$. As $R$ is in time slice $t_x$, this implies $R\not\subset An(Y)_{D_{\bar{X}}}$ and so there is no hedge for the expression $P(Y|do(X))$ which makes it identifiable.

(\textbf{only if}) By Lemma~\ref{lem:dcn_complete}, if $P(Y|do(X))$ is identifiable then there is no hedge for $P(Y|do(X))$ in $D$. By Lemma~\ref{lem:hedge_static} if $D$ has static confounders, a hedge must be within time slice $t_x$. If time slice $t_x$ does not contain two $R$-rooted C-forests $F$ and $F'$ such that $F'\subseteq F$, $F\cap X \neq 0$, $F'\cap X = 0$, then there is no hedge for any set $Y$ so there is no hedge for the expression 
\begin{align*}
P(V_{t_x+1}\cap An(Y)&|V_{t_x-1},do(X)) \\
&=\frac{P(V_{t_x+1}\cap An(Y),V_{t_x-1}|do(X))}{P(V_{t_x-1})}
\end{align*} 
which makes it identifiable. Now let's assume time slice $t_x$ contains two $R$-rooted C-forests $F$ and $F'$ such that $F'\subseteq F$, $F\cap X \neq 0$, $F'\cap X = 0$, then $R\not\subset An(Y)_{D_{\bar{X}}}$ (if $R\subset An(Y)_{D_{\bar{X}}}$ $D$ would contain a hedge by definition). As $R$ is in time slice $t_x$, $R\not\subset An(Y)_{D_{\bar{X}}}$ implies $R\not\subset An(V_{t_x+1}\cap An(Y))_{D_{\bar{X}}}$ and $R\not\subset An(V_{t_x+1}\cap An(Y),V_{t_x-1})_{D_{\bar{X}}}$ so there is no hedge for $P(V_{t_x+1}\cap An(Y)|V_{t_x-1}, do(X))$ which makes this expression identifiable.
\qed\end{proof}

\begin{lemma}
\label{lem:A_generic}
Assume that an expression $P(V'_{t+\alpha}|V_{t},do(X))$ is identifiable for some $\alpha>0$ and $V'_{t+\alpha}\subseteq V_{t+\alpha}$. Let $A$ be the matrix whose entries $A_{ij}$ correspond to the probabilities $P(V'_{t+\alpha} = v_j|V_t = v_i, do(X))$. Then $P(V'_{t+\alpha}|do(X)) = A\,P(V_t|do(X))$.
\end{lemma}
\begin{proof} 
Case by case evaluation of $A$'s entries. 
\qed\end{proof}

\begin{figure}[H]
\hrule\medskip
Function \textbf{cDCN-ID}($Y$,$t_y$, $X$,$t_x$, $G$,$C$,$T$,$P(V_{t_0})$)

INPUT: 
\begin{itemize}
\item DCN defined by a causal graph $G$ 
on a set of variables $V$ and a set $C \subseteq V \times V$ describing causal relations from $V_t$ to $V_{t+1}$ for every $t$
\item transition matrix $T$ representing the probabilities $P(V_{t+1}|V_{t})$ derived from observational data
\item a set $Y$ included in $V_{t_y}$
\item a set $X$ included in $V_{t_x}$
\item distribution $P(V_{t_0})$ at the initial state, 
\end{itemize}

OUTPUT: The distribution  $P(Y|do(X))$ if it is identifiable, or else FAIL

\begin{enumerate}
\item let $G'$ be the acyclic graph formed by joining $G_{t_x-2}$, $G_{t_x-1}$, $G_{t_x}$, and $G_{t_x+1}$
by the causal relations given by $C$;
\item run the standard ID algorithm for expression $P(V_{t_x+1}\cap An(Y)|V_{t_x-1},do(X))$ on $G'$; if it returns FAIL, return FAIL;
\item else, use the resulting distribution to compute the transition matrix $A$, where $A_{ij} = P(V_{t_x+1}\cap An(Y)=v_i|V_{t_x-1}=v_j, do(X))$;
\item let $M_t$ be the matrix $T$ marginalized as $P(V_{t}\cap An(Y) = v_j|V_{t-1}\cap An(Y) = v_i)$
\item return $\left[\prod\limits_{t=t_x+2}^{t_y} M_t\right]A\,T^{t_x-1-t_0}\,P(V_{t_0})$;
\end{enumerate}

\caption{The cDCN algorithm for DCNs with static confounders}
\label{fig:static-algo-complete}
\medskip\hrule
%\smallskip 
%\hrule
\end{figure}

\begin{lemma}
\label{lem:Tafter_marginal}
Let $D$ be a DCN with static confounders. 
Let $X\subseteq V_{t_x}$ and $Y\subseteq V_{t_y}$ for two time slices $t_x < t_y$. Then 
$P(Y|do(X))=\left[\prod\limits_{t=t_x+2}^{t_y} M_t\right]P(V_{t_x+1}\cap An(Y)|do(X))$ where $M_t$ is the matrix whose entries correspond to the probabilities $P(V_{t}\cap An(Y) = v_j|V_{t-1}\cap An(Y) = v_i)$.
\end{lemma}
\begin{proof}
For the identification of $P(Y|do(X))$ we can restrict our attention to the subset of variables in $D$ that are ancestors of Y. Then we repeatedly apply Lemma~\ref{lem:Tafter} on this subset from $t=t_x+2$ to $t=t_y$ until we find $P(V_{t_y}\cap An(Y)|do(X))=P(Y|do(X))$.
\qed\end{proof}

\begin{theorem}
\label{thm:calc_static_complete}
Let $D$ be a DCN with static confounders and transition matrix $T$. 
Let $X\subseteq V_{t_x}$ and $Y\subseteq V_{t_y}$ for two time slices $t_x < t_y$.
If $P(Y|do(X))$ is identifiable then $P(Y|do(X))=\left[\prod\limits_{t=t_x+2}^{t_y} M_t\right]AT^{t_x-1-t_0}P(V_{t_0})$ where $A$ is the matrix whose entries $A_{ij}$ correspond to $P(V_{t_x+1}\cap An(Y)|V_{t_x-1}, do(X))$ and $M_t$ is the matrix whose entries correspond to the probabilities $P(V_{t}\cap An(Y) = v_j|V_{t-1}\cap An(Y) = v_i)$.
\end{theorem}

\begin{proof} 
Applying Lemma~\ref{lem:Tbefore}, we obtain that
$$P(V_{t_x-1}|do(X)) = T^{t_x-1-t_0}P(V_{t_0}).$$
By Lemma~\ref{lem:static_complete} $P(V_{t_x+1}\cap An(Y)|V_{t_x-1}, do(X))$ is identifiable so Lemma~\ref{lem:A_generic} guarantees that $P(V_{t_x+1}\cap An(Y)|do(X)) = A\,P(V_{t_x-1}|do(X)) = A\,T^{t_x-1-t_0}P(V_{t_0})$. Then we apply Lemma~\ref{lem:Tafter_marginal} and obtain the resulting expression $P(Y|do(X))=\left[\prod\limits_{t=t_x+2}^{t_y} M_t\right]AT^{t_x-1-t_0}P(V_{t_0})$.
\qed\end{proof}

The cDCN-ID algorithm for identification of DCNs with static confounders is given in Figure~\ref{fig:static-algo-complete}. 

\begin{theorem}[Soundness and completeness]
\label{thm:completeness}
The cDCN-ID algorithm for DCNs with static confounders is sound and complete.
\end{theorem}

\begin{proof} 
The completeness derives from Lemma~\ref{lem:static_complete} and the soundness from Theorem~\ref{thm:calc_static_complete}.
\qed\end{proof}

\subsection{Complete DCN identification algorithm with Dynamic Confounders}
\label{sec:completesdynamic}

We now discuss the complete identification of DCNs with dynamic confounders. First we introduce the concept of dynamic time span from which we derive two lemmas.

\begin{definition}[Dynamic time span]
Let $D$ be a DCN with dynamic confounders and $X\subseteq V_{t_x}$. Let $t_m$ be the maximal time slice d-connected by confounders to $X$; $t_m-t_x$ is called the dynamic time span of $X$ in $D$.
\end{definition}

Note that the dynamic time span of $X$ in $D$ can be in some cases infinite, the simplest case being when $X$ is connected by a confounder to itself at $V_{t_x+1}$. In this paper we consider finite dynamic time spans only. We will label the dynamic time span of $X$ as $t_{dx}$.

\begin{lemma}
\label{lem:hedge_dynamic}
Let $D$ be a DCN with dynamic confounders and $X$, $Y$ sets of variables in $D$. Let $t_{dx}$ be the dynamic time span of $X$ in $D$. If there is a hedge for $P(Y|do(X))$ in $D$ then the hedge does not include variables at $t>t_x+t_{dx}$.
\end{lemma}

\begin{proof}
By definition of hedge, $F$ and $F'$ are connected by confounders to $X$. The maximal time point connected by confounders to $X$ is $t_x+t_{dx}$.
\qed\end{proof}

\begin{figure}[H]
\hrule\medskip
Function \textbf{cDCN-ID}($Y$,$t_y$, $X$,$t_x$, $G$,$C$,$C'$,$T$,$P(V_{t_0})$)

INPUT: 
\begin{itemize}
\item DCN defined by a causal graph $G$ 
on a set of variables $V$ and a set $C \subseteq V \times V$ describing causal relations from $V_t$ to $V_{t+1}$ for every $t$, and a set $C' \subseteq V \times V$ describing confounder relations from $V_t$ to $V_{t+1}$ for every $t$
\item transition matrix $T$ for $G$ derived
from observational data
\item a set $Y$ included in $V_{t_y}$
\item a set $X$ included in $V_{t_x}$
\item distribution $P(V_{t_0})$ at the initial state, 
\end{itemize}

OUTPUT: The distribution $P(Y|do(X))$ if it is identifiable or else FAIL

\begin{enumerate}
\item let $G'$ be the acyclic graph formed by joining $G_{t_x-2}$, $G_{t_x-1}$, $G_{t_x}$, and $G_{t_x+1}$
by the causal relations given by $C$ and confounders given by $C'$;
\item run the standard ID algorithm for expression $P(V_{t_x+t_{dx}+1}\cap An(Y)|V_{t_x-1}, do(X))$ on $G'$; if it returns FAIL, return FAIL;
\item else, use the resulting distribution to compute the transition matrix $A$, where $A_{ij} = P(V_{t_x+t_{dx}+1}\cap An(Y)=v_i|V_{t_x-1}=v_j, do(X))$;
\item for each $t$ from $t_x+t_{dx}+2$ up to $t_y$:
	\begin{enumerate}
    \item let $G''$ be the causal graph composed of time slices $G_{t_x-1}$, $G_{t_x}$, \dots, $G_{t}$
	\item run the standard ID algorithm on $G''$ for the expression $P(V_t\cap An(Y)|V_{t-1}\cap An(Y),do(X))$; if it returns FAIL, return FAIL;
    \item else, use the resulting distribution to compute the transition matrix $M_t$, where $(M_t)_{ij} = P(V_{t}\cap An(Y)=v_i|V_{t-1}\cap An(Y)=v_j, do(X))$;
	\end{enumerate}
\item return $ \left[\prod\limits_{t=t_x+t_{dx}+2}^{t_y} M_t\right]\,A\,T^{t_x-1-t_0}P(V_{t_0})$;
\end{enumerate}

\caption{The cDCN algorithm for DCNs with dynamic confounders}
\label{fig:dynamic-algo-complete}
\medskip\hrule
%\smallskip 
%\hrule
\end{figure}

\begin{lemma}
\label{lem:dynamic_complete}
Let $D$ be a DCN with dynamic confounders. Let $X\subseteq V_{t_x}$ and $Y\subseteq V_{t_y}$ for two time slices $t_x, t_y$. Let $t_{dx}$ be the dynamic time span of $X$ in $D$ and $t_x + t_{dx} < t_y$.
$P(Y|do(X))$ is identifiable if and only if $P(V_{t_x+t_{dx}+1}\cap An(Y)|V_{t_x-1}, do(X))$ is identifiable.
\end{lemma}

\begin{proof} 
Similarly to the proof of Lemma~\ref{lem:static_complete}, but replacing "static" by "dynamic", $V_{t_x+1}$ by $V_{t_x+t_{dx}+1}$, Lemma~\ref{lem:hedge_static} by Lemma~\ref{lem:hedge_dynamic}, and "time slice $t_x$" by "time slices $t_x$ to $t_x+t_{dx}$".
\qed\end{proof}

\begin{theorem}
\label{thm:calc_dynamic_complete}
Let $D$ be a DCN with dynamic confounders and $T$ be its transition matrix under no interventions. Let $X\subseteq V_{t_x}$ and $Y\subseteq V_{t_y}$ for two time slices $t_x, t_y$. Let $t_{dx}$ be the dynamic time span of $X$ in $D$ and $t_x + t_{dx} < t_y$.
If $P(Y|do(X))$ is identifiable then:
\begin{enumerate}
\item $P(V_{t_x+t_{dx}+1}\cap An(Y)|V_{t_x-1}, do(X))$ is identifiable by matrix $A$ 
\item For all
$t > t_x+t_{dx}+1$, $P(V_{t}\cap An(Y)|V_{t-1}\cap An(Y), do(X))$ is identifiable by matrix $M_t$
\item $P(Y|do(X))=\left[\prod\limits_{t=t_x+t_{dx}+2}^{t_y} M_t\right]\,A\,T^{t_x-1-t_0}P(V_{t_0})$
\end{enumerate}
\end{theorem}
\begin{proof} 
We obtain the first statement from Lemma~\ref{lem:dynamic_complete} and Lemma~\ref{lem:A_generic}. Then if $t > t_x+t_{dx}+1$ the set $(V_{t}\cap An(Y),V_{t-1}\cap An(Y))$ has the same ancestors than $Y$ within time slices $t_x$ to $t_x+t_{dx}+1$, so if $P(Y|do(X))$ is identifiable then $P(V_{t}\cap An(Y)|V_{t-1}\cap An(Y), do(X))$ is identifiable, which proves the second statement. Finally we obtain the third statement similarly to the proof of Theorem~\ref{thm:dynamic} but using statements 1 and 2 as proved instead of assumed.
\qed\end{proof}

The cDCN-ID algorithm for DCNs with dynamic confounders is given in Figure~\ref{fig:dynamic-algo-complete}.

\begin{theorem}[Soundness and completeness]
\label{thm:completeness_dynamic}
The cDCN-ID algorithm for DCNs with dynamic confounders is sound and complete.
\end{theorem}

\begin{proof} 
The completeness derives from the first and second statements of Theorem~\ref{thm:calc_dynamic_complete}. The soundness derives from the third statement of Theorem~\ref{thm:calc_dynamic_complete}.
\qed\end{proof}

\section{Transportability in DCN}
\label{sec:transp}
\cite{pearl2011transportability} introduced the sID algorithm, based on do-calculus, to identify a transport formula between two domains, where the effect in a target domain can be estimated from experimental results in a source domain and some observations on the target domain, thus avoiding the need to perform an experiment on the target domain.

Let us consider a country with a number of alternative roads linking city pairs in different provinces. Suppose that the alternative roads are all consistent with the same causal model (such as the one in Figure~\ref{fig:dcn_confounder_compact}, for example) but have different traffic patterns (proportion of cars/trucks, toll prices, traffic light durations...).
Traffic authorities in one of the provinces may have experimented with policies and observed the impact on, say, traffic delay. This information may be usable to predict the average travel delay in another province for a given traffic policy. The source domain (province where the impact of traffic policy has already been monitored) and target domain (new province) share the same causal relations among variables, represented by a single DCN (see Figure~\ref{fig:dcntransport}). 

The target domain may have specific distributions of the toll price and traffic signs, which are accounted for in the model by adding a set of selection variables to the DCN, pointing at variables whose distribution differs among the two domains. If the DCN with the selection variables is identifiable for the traffic delay upon increasing the toll price, then the DCN identification algorithm provides a transport formula which combines experimental probabilities from the source domain and observed distributions from the target domain. Thus the traffic authorities in the new province can evaluate the impacts before effectively changing traffic policies. This amounts to relational knowledge transfer learning between the two domains \cite{pan2010survey}.

\begin{figure}[H]
\begin{center}
\includegraphics[width=2.8cm]{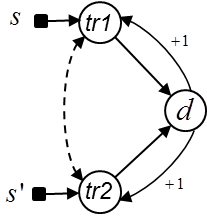}
\end{center}
\caption{A DCN with selection variables $s$ and $s'$, representing the differences in the distribution of variables $tr1$ and $tr1$ in two domains $M_1$ and $M_2$ (two provinces in the same country). This model can be used to evaluate the causal impacts of traffic policy in the target domain $M_2$ based on the impacts observed in the source domain $M_1$.}
\label{fig:dcntransport}
\end{figure}

Consider a DCN with static confounders only. We have demonstrated already that for identification of the effects
of an intervention at time $t_x$ we can restrict our attention to four time slices of the DCN, $t_x-2$, $t_x-1$, $t_x$, and $t_x+1$. Let $M_1$ and $M_2$ be two domains based on this same DCN, 
though the distributions of some variables in $M_1$ and $M_2$ may differ. Then we have 
\begin{align*}
P_{M_2}(Y|do(X))=T_{M_2}^{t_y-(t_x+1)}A_{M_2}T_{M_2}^{t_x-1-t_0}P(V_{t_0}),
\end{align*}
where the entry $ij$ of matrix $A_{M_2}$ corresponds to the transition probability $P_{M_2}(V_{t_x+1}=v_i|V_{t_x-1}=v_j, do(X))$.

%$$\frac{P_{M_2}(V_{t_x+1}=i,V_{t_x-1}=j|do(X))}{P_{M_2}(V_{t_x-1}=j)}$$

% $$
% \alpha_{M_2}^{jk}=\frac{1}{P_{M_2}(V_{t_x-1}=v_k)}\beta_{M_2}^{jk}
% $$
% and 
% $$
% \beta_{M_2}^{jk}=P_{M_2}(V_{t_x+1}=v_j,V_{t_x-1}=v_k|do(X)).
% $$
By applying the identification algorithm sID, with selection variables, to the elements of matrix $A$ we then obtain a transport formula, which combines experimental distributions in $M_1$ with observational distributions in $M_2$. The algorithm for transportability of causal effects with static confounders is given in Figure~\ref{fig:transport-static-algo}.

\begin{figure}[H]
\hrule\medskip
Function \textbf{DCN-sID}($Y$,$t_y$, $X$,$t_x$, $G$,$C$,$T_{M_2}$,$P_{M_2}(V_{t_0})$,$I_{M_1}$)

INPUT: 
\begin{itemize}
\item DCN defined by a causal graph $G$ (common to both source and target domains $M_1$ and $M_2$) over a set of variables $V$ and a set $C \subseteq V \times V$ describing causal relations from $V_t$ to $V_{t+1}$ for every $t$
\item transition matrix $T_{M_2}$ for $G$ derived
from observational data in $M_2$
\item a set $Y$ included in $V_{t_y}$
\item a set $X$ included in $V_{t_x}$
\item distribution $P_{M_2}(V_{t_0})$ at the initial state in $M_2$
\item set of interventional distributions $I_{M_1}$ in $M_1$
\item set S of selection variables

\end{itemize}

OUTPUT: The distribution  $P_{M_2}(Y|do(X))$ in $M_2$ in terms of $T_{M_2}$, $P_{M_2}(V_{t_0})$ and $I_{M_1}$, or else FAIL

\begin{enumerate}
\item let $G'$ be the acyclic graph formed by joining $G_{t_x-2}$, $G_{t_x-1}$, $G_{t_x}$, and $G_{t_x+1}$
by the causal relations given by $C$;
\item run the standard sID algorithm for expression $P(V_{t_x+1}|V_{t_x-1},do(X))$ on $G'$; if it returns FAIL, return FAIL;
\item else, use the resulting transport formula to compute the transition matrix $A$, where $A_{ij} = P(V_{t_x+1}=v_i|V_{t_x-1}=v_j, do(X))$;
\item return $\sum_{V_{t_y}\setminus Y} T^{t_y-(t_x+1)}\,A\,T^{t_x-1-t_0}\,P(V_{t_0})$;
\end{enumerate}

\caption{The DCN-sID algorithm for the transportability in DCNs with static confounders}
\label{fig:transport-static-algo}
\medskip\hrule
%\smallskip 
%\hrule
\end{figure}

For brevity we omit the algorithm extension to dynamic confounders, and the completeness results, which follow the same confounder caveats already explained in the previous sections.

\section{Experiments}
\label{sec:experiments}
In this section we provide some numerical examples of causal effect identifiability in DCN, using the algorithms proposed in this paper.

In our first example, the DCN in Figure~\ref{fig:dcn_confounder_compact} represents how  the traffic between two cities evolves. There are two roads and drivers choose every day to use one or the other road. Traffic conditions on either road on a given day ($tr1$, $tr2$) affect the travel delay between the cities on that same day ($d$). Driver experience influences the road choice next day, impacting $tr1$ and $tr2$. For simplicity we assume variables $tr1$, $tr2$ and $d$ to be binary. Let's assume that from Monday to Friday the joint distribution of the variables follow transition matrix $T_1$ while on Saturday and Sunday they follow transition matrix $T_2$. These transition matrices indicate the traffic distribution change from the previous day to the current day. This system is a DCN with static confounders, and has a markov chain representation as in Figure~\ref{fig:dcn_confounder_compact}.

\[ T_1 = \left( \begin{matrix}
0.0&0.4&0.0&0.3&0.0&0.2&0.0&0.1 \\
0.0&0.4&0.0&0.3&0.0&0.2&0.0&0.1 \\
0.0&0.4&0.0&0.3&0.0&0.2&0.0&0.1 \\
0.0&0.4&0.0&0.3&0.0&0.2&0.0&0.1 \\
0.2&0.0&0.0&0.1&0.4&0.0&0.0&0.3 \\
0.2&0.0&0.0&0.1&0.4&0.0&0.0&0.3 \\
0.2&0.0&0.0&0.1&0.4&0.0&0.0&0.3 \\
0.2&0.0&0.0&0.1&0.4&0.0&0.0&0.3 \\
\end{matrix} \right)\] 

\[ T_2 = \left( \begin{matrix}
0.1&0.0&0.3&0.1&0.2&0.2&0.0&0.1 \\
0.1&0.0&0.3&0.1&0.2&0.2&0.0&0.1 \\
0.1&0.0&0.3&0.1&0.2&0.2&0.0&0.1 \\
0.1&0.0&0.3&0.1&0.2&0.2&0.0&0.1 \\
0.0&0.2&0.1&0.0&0.1&0.3&0.3&0.0 \\
0.0&0.2&0.1&0.0&0.1&0.3&0.3&0.0 \\
0.0&0.2&0.1&0.0&0.1&0.3&0.3&0.0 \\
0.0&0.2&0.1&0.0&0.1&0.3&0.3&0.0 \\
\end{matrix} \right)\] 

The average travel delay $d$ during a two week period is shown in Figure~\ref{fig:chart_exp0}. 

\begin{figure}[H]
\begin{center}
\includegraphics[width=7cm]{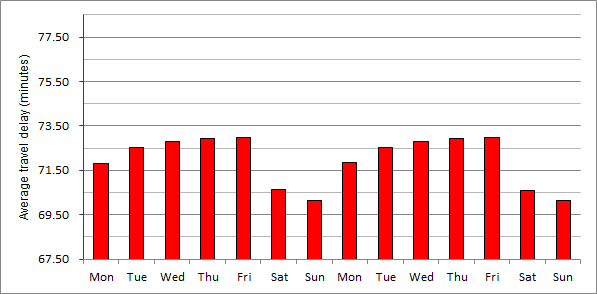}
\end{center}
\caption{Average travel delay of the DCN without intervention.}
\label{fig:chart_exp0}
\end{figure}

Now let's perform an intervention by altering the traffic on the first road $tr1$ and evaluate the subsequent evolution of the average travel delay $d$. We use the algorithm for DCNs with static confounders. We trigger line 1 of the DCN-ID algorithm in Figure~\ref{fig:static-algo-complete} and build a graph consisting of four time slices $G'=(G_{t_x-2},G_{t_x-1},G_{t_x},G_{t_x+1})$ as shown in Figure~\ref{fig:fig_example}.

\begin{figure}[H]
\begin{center}
\includegraphics[width=7cm]{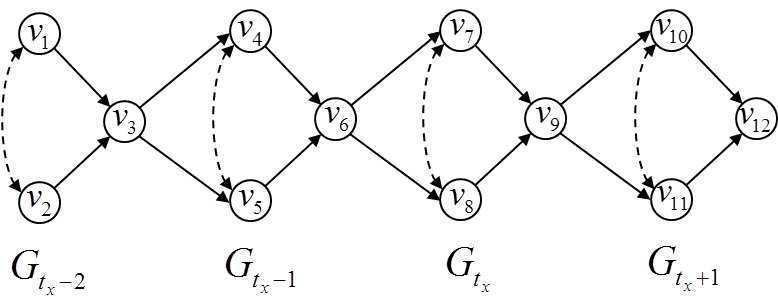}
\end{center}
\caption{Causal graph $G'$ consisting of four time slices of the DCN, from $t_x-2$ to $t_x+1$}
\label{fig:fig_example}
\end{figure}

The ancestors of any future delay at
$t=t_y$ are all the variables in the DCN up to $t_y$, so in line 2 we run the standard ID algorithm for $\alpha=P(v_{10},v_{11},v_{12}|v_4,v_5,v_6, do(v_7))$ on $G'$, which returns the expression $\alpha$:
\begin{align*}
\sum_{v_1,v_2,v_3,v_8,v_9}\frac{P(v_1,v_2,...v_{12})\sum_{v_7,v9}P(v_7,v_8,v_9|v_4,v_5,v_6)}{P(v_4,v_5,v_6)\sum_{v_9}P(v_7,v_8,v_9|v_4,v_5,v_6)}
\end{align*}

Using this expression, line 3 of the algorithm computes the elements of matrix $A$. If we perform the intervention on a Thursday the matrices $A$ for $v_7=0$ and $v_7=1$ can be evaluated from $T_1$.

\[ A_{v_7=0} = \left( \begin{matrix}
0.0&0.4&0.0&0.3&0.0&0.2&0.0&0.1 \\
0.0&0.4&0.0&0.3&0.0&0.2&0.0&0.1 \\
0.0&0.4&0.0&0.3&0.0&0.2&0.0&0.1 \\
0.0&0.4&0.0&0.3&0.0&0.2&0.0&0.1 \\
0.0&0.4&0.0&0.3&0.0&0.2&0.0&0.1 \\
0.0&0.4&0.0&0.3&0.0&0.2&0.0&0.1 \\
0.0&0.4&0.0&0.3&0.0&0.2&0.0&0.1 \\
0.0&0.4&0.0&0.3&0.0&0.2&0.0&0.1 \\
\end{matrix} \right)\]
\[ A_{v_7=1} = \left( \begin{matrix}
0.2&0.0&0.0&0.1&0.4&0.0&0.0&0.3 \\
0.2&0.0&0.0&0.1&0.4&0.0&0.0&0.3 \\
0.2&0.0&0.0&0.1&0.4&0.0&0.0&0.3 \\
0.2&0.0&0.0&0.1&0.4&0.0&0.0&0.3 \\
0.2&0.0&0.0&0.1&0.4&0.0&0.0&0.3 \\
0.2&0.0&0.0&0.1&0.4&0.0&0.0&0.3 \\
0.2&0.0&0.0&0.1&0.4&0.0&0.0&0.3 \\
0.2&0.0&0.0&0.1&0.4&0.0&0.0&0.3 \\
\end{matrix} \right)\] 

In line 4 we find that transition matrices $M_t$ are the same than for the DCN without intervention. Figure~\ref{fig:chart_exp1} shows the average travel delay without intervention, and with intervention on the traffic conditions of the first road.
 
\begin{figure}[H]
\begin{center}
\includegraphics[width=7cm]{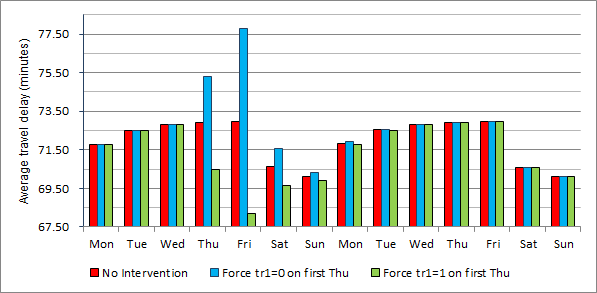}
\end{center}
\caption{Average travel delay of the DCN without intervention, and with interventions $tr1=0$ and $tr1=1$ on the first Thursday}
\label{fig:chart_exp1}
\end{figure}

In a second numerical example, we consider that the system is characterized by a unique transition matrix $T$ and the delay $d$ tends to a steady state. We measure $d$ without intervention and with intervention on $tr1$ at $t=15$. The system's transition matrix $T$ is shown below:

\[ T = \left( \begin{matrix}
0.02&0&0.03&0&0.26&0.13&0.34&0.22 \\
0.02&0&0.03&0&0.26&0.13&0.34&0.22 \\
0.02&0&0.03&0&0.26&0.13&0.34&0.22 \\
0.02&0&0.03&0&0.26&0.13&0.34&0.22 \\
0.34&0.1&0.24&0.21&0&0.02&0.09&0 \\
0.34&0.1&0.24&0.21&0&0.02&0.09&0 \\
0.34&0.1&0.24&0.21&0&0.02&0.09&0 \\
0.34&0.1&0.24&0.21&0&0.02&0.09&0 \\
\end{matrix} \right)\] 

Figure~\ref{fig:chart_exp2} shows the evolution of $d$ with no intervention and with intervention. 

\begin{figure}[H]
\begin{center}
\includegraphics[width=7cm]{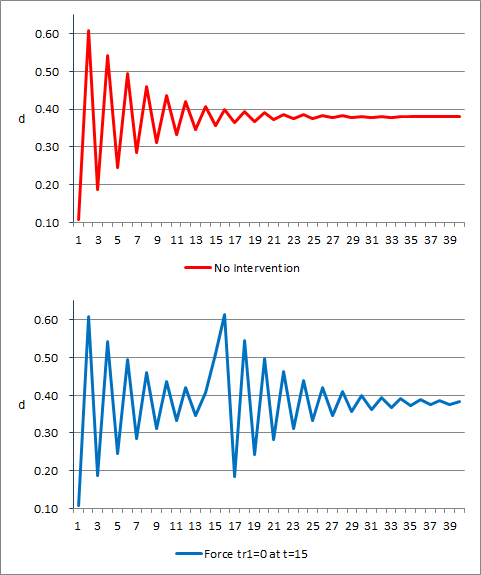}
\end{center}
\caption{Average $d$ of the DCN without intervention and with intervention on $tr1$ at $t=15$.}
\label{fig:chart_exp2}
\end{figure}

As shown in the examples, the DCN-ID algorithm calls ID only once with a graph of size $4|G|$ and evaluates the elements of matrix A with complexity $O((4k)^{(b+2)}$, where $k=3$ is the number of variables per slice and $b=1$ is the number of bits used to encode the variables. The rest is the computation of transition matrix multiplications, which can be done with complexity $O(n.b^2)$, with $n=40-15$ in example 2. To obtain the same result with the ID algorithm by brute force, we would require processing $n$ times the identifiability of a graph of size $40|G|$, with overall complexity $O((k)^{(b+2)}+(2k)^{(b+2)}+(3k)^{(b+2)}+...+(n.k)^{(b+2)})$.

\section{Conclusions and Future Work}
This paper introduces dynamic causal networks and their analysis with do-calculus, so far studied thoroughly only in static causal graphs. We extend the ID algorithm to the identification of DCNs, and remark the difference between static vs.\ dynamic confounders.  We also provide an algorithm for the transportability of causal effects from one domain to another with the same dynamic causal structure.

For future work, note that in the present paper we have assumed all intervened variables to be in the same time slice; removing this restriction may have some moderate interest. We also want to extend the introduction of causal analysis to a number of dynamic settings, including Hidden Markov Models, and study properties of DCNs in terms of Markov chains (conditions for ergodicity, for example). Finally, evaluating the distribution returned by ID is in general unfeasible (exponential in the number of variables and domain size); identifying tractable sub-cases or feasible heuristics is a general question in the area.

\begin{acknowledgements}
Research was partially funded by SGR2014-890 (MACDA) project of
the Generalitat de Catalunya and MINECO project APCOM (TIN2014-57226-
P).
\end{acknowledgements}

\bibliographystyle{spphys}       % APS-like style for physics
\bibliography{dcn}   % name your BibTeX data base

\end{document}